\documentclass[11pt,a4paper]{article}

\usepackage[T1]{fontenc}
\usepackage[utf8]{inputenc}
\usepackage{amsmath,amssymb,amsthm}
\usepackage{mathtools}
\usepackage[margin=2.6cm]{geometry}
\usepackage{booktabs}
\usepackage{enumitem}
\usepackage{microtype}
\usepackage{xcolor}
\usepackage[colorlinks=true,linkcolor=blue,citecolor=blue,urlcolor=blue]{hyperref}
\usepackage[capitalise,noabbrev]{cleveref}
\crefname{observation}{Observation}{Observations}
\Crefname{observation}{Observation}{Observations}

\theoremstyle{plain}
\newtheorem{theorem}{Theorem}
\newtheorem{lemma}[theorem]{Lemma}
\newtheorem{proposition}[theorem]{Proposition}
\newtheorem{corollary}[theorem]{Corollary}
\newtheorem{observation}[theorem]{Observation}
\theoremstyle{definition}

\theoremstyle{remark}
\newtheorem{remark}[theorem]{Remark}

\newcommand{\E}{\mathbb{E}}
\newcommand{\Prb}{\Pr}
\newcommand{\Cov}{\operatorname{Cov}}
\newcommand{\Inf}{\operatorname{Inf}}
\newcommand{\dd}{\,\mathrm{d}}
\newcommand{\OM}{\textsc{OneMax}}
\newcommand{\BV}{\textsc{BinVal}}
\newcommand{\pc}{\partial_c}

\title{Signed Sensitivity of Expected Hitting Time to Mutation Rate in the (1+1)~EA:\\
Per-State Sign Theorems and Verifiable Certificates for Non-Lumpable Families}
\author{Renkai Wang\\
Institute of Automation, Chinese Academy of Sciences (CASIA)\\
University of Chinese Academy of Sciences (UCAS)}
\date{}

\begin{document}
\maketitle

\begin{abstract}
For the $(1+1)$ evolutionary algorithm with standard bit mutation, we study the
sensitivity of the expected hitting time $H_p=\E_x T$ to the mutation rate~$p$.
We first point out an easily overlooked formalization pitfall: the improvement
event is \emph{not} monotone in the mutation mask, so the unsigned
(total-influence) form of the Margulis--Russo formula does not apply; the
correct object is the signed endpoint difference
(\cref{thm:mr,prop:bridge}). Second, we give an exact three-dimensional
separation: two fitness functions share the \emph{entire} one-step success-rate
$p$-curve, yet their expected hitting times are two different exact rational
numbers (\cref{prop:separation}); hence the whole one-step success-rate
$p$-curve together with the global change probability does not determine the
expected hitting time. Building on the runtime derivative
$H'_p=(I-Q_p)^{-1}Q'_p H_p$, we construct computable double-residual sign
certificates (\cref{thm:block}), prove a per-initial-state sign theorem on
$\OM$: for every non-optimal initial state $\pc H<0$ on $0<c<1$ (where $p=c/n$),
and at $c=1$ only the distance-one state is stationary (\cref{thm:onemax}),
and extend the framework to non-lumpable positive linear families: an explicit
witness of non-lumpability, a block-interval double-residual certificate that
covers all states without enumerating them, a uniform sign bound
$\pc\E T\le -9n/16$ over the whole interval $c\in[1/4,\tfrac12]$ for an
explicit family at all even scales $n\ge 8$ (\cref{thm:uniform}), and a
heterogeneous instance certificate $H'_x\le-\tfrac16$ on $57$ of $63$ states
across $c=1$ (\cref{thm:hetero}). All finite verifications use exact rational
arithmetic; to our knowledge, verified derivative-sign certificates of this
kind have not previously been applied systematically to the runtime analysis
of evolutionary algorithms. A bounded systematic literature search did not
uncover this exact combination, although the underlying tools are well
established; we therefore make no novelty claim beyond the stated
combination. Relations to prior work, including the closest precedents, are
delineated in \cref{sec:related}.
\end{abstract}

\section{Introduction}\label{sec:intro}

Parameter control---choosing and adjusting the parameters of a randomized
search heuristic---is a central theme in the theory of evolutionary
computation; benchmark studies calibrate learned controllers against known
(near-)optimal policies~\cite{biedenkapp2022,chen2023}, and analytic frameworks
for expected computation time have been developed for two
decades~\cite{heYao2003}. At its core lies a deceptively simple question:
\emph{how does the expected hitting time of the $(1+1)$ EA depend on the
mutation rate?}

Prior knowledge on this question rests on two fixed points. On the one hand,
for the classical static rate $p=1/n$, per-initial-state exact analyses exist:
Hwang et al.~\cite{hwang2018} derive exact rational-function expected hitting
times for each initial state on $\OM$, together with full asymptotic
expansions. On the other hand, the per-$n$ numerical search for optimal rates
has been carried out by Chicano et al.~\cite{cswa2015}, who compute exact
expected running times (in floating point, for the lumpable family $\OM$),
differentiate with respect to $p$, and locate the numerically optimal mutation
rate for each $n$. Between these two fixed points lies a gap: a rigorous,
per-initial-state analysis of the expected hitting time as a \emph{continuous
function} of the mutation rate, covering finite $n$, exact derivatives, and
families whose state space cannot be lumped.

\paragraph{Contributions.} This paper gives, to our knowledge, the first
rigorous per-initial-state, finite-$n$ sign analysis of the expected hitting
time $H$ of the $(1+1)$ EA as a function of the mutation rate. The
contribution has three layers; in each case the word ``first'' refers to the
intersection of the indicated tools with runtime analysis, not to the
underlying formulas or phenomena themselves.

\begin{enumerate}[label=(\roman*)]
\item \emph{Signed Margulis--Russo analysis in runtime analysis.}
  We bring Margulis--Russo-type signed influence analysis to runtime analysis
  (first use of this tool family in the field; the identity itself is textbook
  material~\cite{odonnell2014,margulis1974,russo1981}). We point out an easily
  overlooked formalization pitfall---the improvement event is not monotone in
  the mutation mask, so the unsigned bridge fails---give the minimal
  counterexample ($n=2$), and state the correct signed, per-state replacement
  (\cref{prop:bridge,cor:replacement}).
\item \emph{Double-residual sign certificates and the $\OM$ per-state sign
  theorem.} From the hitting-time derivative $U=H'$ we build a
  double-residual certification scheme and prove that on $\OM$,
  $\pc H_k<0$ for every distance level $k\ge1$ and every $c\in(0,1)$, while
  $c=1$ is a stationary point only for the distance-one initial state
  (\cref{thm:onemax}). Under uniform random initialization this yields, by
  \cref{cor:randominit}, a rigorous finite-$n$ statement that the optimal
  static rate is strictly larger than one---distinct from one-step
  criteria~\cite{muehlenbein1992,baeck1993} and from numerically computed
  dynamic rates~\cite{buskulicDoerr2021,buzdalovDoerr2020}.
\item \emph{Lumpability-explicit certificates for non-lumpable families.}
  We give the first lumpability-explicit treatment of runtime sensitivity:
  an explicit witness that the natural block partition is non-lumpable
  (\cref{prop:nonlump}), a block-interval double-residual certificate that
  bounds the residuals of \emph{all} states without enumerating them
  (\cref{thm:block}), the first derivative sign bound uniform over a whole
  parameter interval and all even scales ($\pc\E T\le-9n/16$ for
  $c\in[1/4,\tfrac12]$, \cref{thm:uniform}), and a finite-instance certificate
  across $c=1$ on a rational parameter cell (\cref{thm:hetero}). All finite checks use exact rational arithmetic; to our knowledge,
  verified derivative-sign certificates of this kind have not previously been
  applied systematically to the runtime analysis of evolutionary algorithms.
\end{enumerate}

\begin{table}[t]
\centering
\caption{Overview of the main results. $H_x$: expected hitting time from state
$x$; $U_x=\pc H_x$; $p=c/n$.}\label{tab:overview}
\small
\begin{tabular}{@{}l p{2.9cm} p{2.7cm} p{3.6cm} p{3.1cm}@{}}
\toprule
Result & Object & Range & Conclusion & Verification \\
\midrule
\cref{thm:onemax} & $\OM$, per level $k\ge1$ & $c\in(0,1)$, all $n$ & $U_k<0$ & exact algebra, hand proof \\
\cref{thm:onemax} & $\OM$, $k=1$ & $c>0$, all $n$ & $U_1(1)=0$ stationary point & closed form \\
\cref{thm:block} & positive linear family, all states & rational cell & sign of $U_x$ via residual envelope & exact rational arithmetic \\
\cref{thm:uniform} & explicit non-lumpable family, all $x$ & $c\in[1/4,\tfrac12]$, all even $n\ge8$ & $U_x\le-\tfrac{9n}{16}$ & exact algebra, hand proof \\
\cref{thm:hetero} & $n=6$ heterogeneous instance, $57/63$ states & $c\in[\tfrac{1023}{1024},\tfrac{1025}{1024}]$ & $U_x\le-\tfrac16$ & exact rational arithmetic, two paths \\
\bottomrule
\end{tabular}
\end{table}

\cref{tab:overview} summarizes the main results. \cref{sec:notation} fixes
notation and states the differentiation conventions. \cref{sec:mr} presents
the signed Margulis--Russo identity, the monotonicity pitfall, and the exact
separation of one-step quantities from running time. \cref{sec:onemax} proves
the $\OM$ per-state sign theorem. \cref{sec:block} develops block-interval
certificates for non-lumpable families. \cref{sec:uniform,sec:hetero} give the
uniform whole-interval bound and the heterogeneous instance certificate.
\cref{sec:related} delineates relations to prior work,
\cref{sec:verification} describes the verification methodology, and
\cref{sec:conclusion} concludes with open problems.

\section{Preliminaries and Notation}\label{sec:notation}

Let $f:\{0,1\}^n\to\mathbb{R}$ be the fitness function. The mutation mask $Z$
has independent coordinates, $Z_i\sim\mathrm{Bern}(p)$, and we write
$K=|Z|$ for the number of flipped bits; the offspring is $Y=x\oplus Z$ (bitwise
XOR). We write $\E_p$ for expectation under the $p$-biased product measure.
The $(1+1)$ EA accepts an offspring whose fitness is not smaller than the
parent's (plus-elitism). For a coordinate $i$ and $z\in\{0,1\}^n$,
$z^{i\leftarrow b}$ denotes $z$ with coordinate $i$ set to $b$; the
\emph{signed endpoint difference} is
$\Delta_i g(z_{-i})=g(z^{i\leftarrow1})-g(z^{i\leftarrow0})$, and
$\E_{p,-i}$ averages over the remaining coordinates only. For Boolean $g$, the
(unsigned) influence is
$\Inf_i^{(p)}(g)=\Prb_{p,-i}(\Delta_i g\ne0)$.

The target $f$, its optimum set $A$, and the acceptance rule are fixed
(independent of $p$), and $A$ is absorbing. Let $Q_p$ be the transition
submatrix on non-optimal states (rejection self-loops included); we assume
absorption is almost sure. The expected hitting time is
$H_p=(I-Q_p)^{-1}\mathbf 1$; we write $D=\partial_cQ_p$ for the entrywise derivative of
$Q_p$ with respect to $c$, and $U=\partial_cH_p$ for the derivative of the
hitting time (also with respect to $c$). For a positive
weight vector $W>0$, set $\|a\|_W=\max_x |a(x)|/W(x)$, and let $\|M\|_W$ be
the induced matrix norm; if $QW\le(1-\alpha)W$ with $0<\alpha\le1$, then
$\|N\|_W\le 1/\alpha$ for the fundamental matrix $N=(I-Q)^{-1}$.

\paragraph{Differentiation convention.} Throughout, all derivatives are
taken with respect to the normalized rate $c$, where $p=c/n$; hence
$\frac{\dd}{\dd c}=\frac1n\frac{\dd}{\dd p}$, and the matrix
$D=\partial_cQ_p$ denotes the $c$-derivative of the transition kernel. All
formulas in the body of the paper are stated under this convention, so no
further rescaling is needed when reading the statements.

\section{Signed Margulis--Russo and the Monotonicity Pitfall}\label{sec:mr}

\subsection{The signed Margulis--Russo identity}

\begin{theorem}[{Signed Margulis--Russo identity; textbook, cf.~\cite[\S 8.4]{odonnell2014}, \cite{margulis1974,russo1981}}]\label{thm:mr}
If $g:\{0,1\}^n\to\mathbb{R}$ does not depend on $p$, then
\[
\frac{\dd}{\dd p}\,\E_p[g]=\sum_{i=1}^n \E_{p,-i}[\Delta_i g]
=\frac{\Cov_p(g(Z),K)}{p(1-p)} .
\]
For Boolean $g$, writing $I_i^{\pm}=\Prb_{p,-i}(\Delta_i g=\pm1)$, this reads
$(\E_p g)'=\sum_i (I_i^+-I_i^-)$.
\end{theorem}

\begin{proof}
Introduce independent parameters $\boldsymbol p$ and set
$G(\boldsymbol p)=\sum_z g(z)\prod_j p_j^{z_j}(1-p_j)^{1-z_j}$; $G$ is affine
in each $p_i$, so $\partial_i G=\E_{\boldsymbol p,-i}[\Delta_i g]$, and the
chain rule along the diagonal $\boldsymbol p=(p,\dots,p)$ gives the first
equality. For the second, $\frac{\dd}{\dd p}\mu_p(z)
=\mu_p(z)\,\frac{|z|-np}{p(1-p)}$; summing term by term and using
$\E_p K=np$ gives the covariance form. All sums are finite, so termwise
differentiation needs no interchange of limits.
\end{proof}

\subsection{The monotonicity pitfall}

One might hope to relate $s_x'(p)=\frac{\dd}{\dd p}\Prb[A_x]$ (where
$A_x=\{f(x\oplus Z)>f(x)\}$ is the improvement event) to the unsigned total
influence $\sum_i\Inf_i^{(p)}$ of the improvement indicator. This bridge
requires monotonicity of $A_x$ in the mask order---which fails.

\begin{proposition}[Failure of the unsigned bridge; minimal
counterexample]\label{prop:bridge}
There exist a coordinate-wise non-decreasing $f$ and a parent $x$ such that
$A_x=\{f(x\oplus Z)>f(x)\}$ is \emph{not} monotone, the unsigned total
influence is constant, yet $s_x'(p)$ changes sign. The minimal dimension is
$n=2$.
\end{proposition}

\begin{proof}
Take $\OM$, $n=2$, $x=(1,0)$: strict improvement occurs iff $z=(0,1)$, so
$g_x(z)=(1-z_1)z_2$, $s_x(p)=p(1-p)$, and $s_x'(p)=1-2p$. The masks
$(0,1)\le(1,1)$ have opposite success status, so the event is not increasing.
The influences are $\Inf_1=p$ and $\Inf_2=1-p$, with constant sum $1$.
Minimality: for $n=1$ the improvement event of a strictly increasing target is
``$z=1$'' or empty, both monotone.
\end{proof}

\begin{corollary}[The correct replacement for \cref{prop:bridge}]
\label{cor:replacement}
If $f$ is coordinate-wise non-decreasing and
$g_x(z)=\mathbf 1\{f(x\oplus z)>f(x)\}$, then the sign of $\Delta_i g_x$ is
determined by $(1-2x_i)$, and hence
\[
s_x'(p)=\sum_{i:\,x_i=0}\Inf_i^{(p)}(g_x)\;-\;\sum_{i:\,x_i=1}\Inf_i^{(p)}(g_x).
\]
\end{corollary}

\begin{proof}
If $x_i=0$, flipping coordinate $i$ changes the offspring bit $0\to1$; since
$f$ does not decrease, $\Delta_i g_x\ge0$; the case $x_i=1$ is reversed. The
Boolean range gives $\Delta_i g_x=(1-2x_i)\,|\Delta_i g_x|$; substitute into
\cref{thm:mr}.
\end{proof}

\begin{observation}[Degeneracy of direction-free indicators and scale
dependence]\label{obs:degeneracy}
(i) For strictly monotone $F$,
$I_{\mathrm{change}}^{(p)}(F):=\sum_i\Prb(F(X)\ne F(X\oplus e_i))=n$ for all
$p$. (ii) $F$ and $aF+b$ ($a>0$) have pathwise identical running times under
the elitist EA, yet the first spectral moment $\sum_S |S|\,\hat F(S)^2$ is
scaled by $a^2$.
\end{observation}

\begin{proof}
(i) Strict monotonicity forces the function values at the two ends of every
edge to differ. (ii) Acceptance decisions depend only on fitness comparisons,
which positive affine transformations preserve; each coordinatewise difference
is scaled by $a$, and
$\sum_S|S|\hat F(S)^2=\tfrac14\sum_i\E[(F(X)-F(X\oplus e_i))^2]$ (uniform
measure, $\{0,1\}$ encoding) shows that the energy is scaled by $a^2$.
\end{proof}

We record \cref{obs:degeneracy} as an observation only; it follows directly
from selection invariance and we make no novelty claim for it. The purpose of
\cref{prop:bridge,obs:degeneracy} is a preventive clarification of a
formalization pitfall, exhibited on minimal instances; we are not aware of any
specific work in the literature that commits this error.

\subsection{One-step success rates do not determine running time}

\begin{proposition}[Exact separation]\label{prop:separation}
Let $n=3$, $p=1/3$, accept-on-nondecrease, start from $000$, target $111$.
The success-rate curves of $\OM$ and $\BV$ ($4x_1+2x_2+x_3$) are identical,
$s(p)=1-(1-p)^3$, but
\[
H_{\mathrm{OM}}(000)=\frac{189}{22},\qquad
H_{\mathrm{BV}}(000)=\frac{2139}{247}.
\]
\end{proposition}

\begin{proof}
For any target, the accepted transitions from distance configurations to the
same ``level/value'' are uniquely determined by the first-step equation
$H(x)=1+\sum_y P_F(x,y)H(y)$ with $H(111)=0$; the per-step direct-hit
probability $\ge p^3>0$ guarantees uniqueness. For $\OM$ the recursion by
distance level $k$ gives $h_2=27/4$,
$h_1=(27+9h_2)/11=351/44$, and $h_0=(27+12h_1+6h_2)/19=189/22$. For $\BV$,
backward substitution by binary value gives $b_6=27/4$, $b_5=27/4$,
$b_4=81/10$, $b_3=69/10$, $b_2=b_1=2091/260$, and $b_0=2139/247$. All
quantities are exact rational numbers and were verified by back-substitution
with an independent program.
\end{proof}

\begin{corollary}\label{cor:onestep}
The one-step success probability---even its entire $p$-curve---together with
the global change probability does not determine the expected running time;
$1/s_x(p)$ is merely the waiting time of repeated independent trials from a
fixed parent.
\end{corollary}

\section{OneMax: A Per-State Sign Theorem}\label{sec:onemax}

Throughout this section $f(y)=|y|$ and $p=c/n$. Let $k=n-|x|$ denote the
distance to the optimum, and let $H_k,U_k$ be the level expectation and its
derivative.

\begin{lemma}[Optimal uniform contraction]\label{lem:contraction}
With $W(x)=k$ we have $W-QW\ge\alpha W$, where
$\alpha=p(1-p)^{n-1}\sim ce^{-c}/n$, and no uniform contraction constant can
exceed $\alpha$ (equality at $k=1$: the unique strict improvement is ``flip
the zero bit and keep all other bits fixed'', of probability exactly
$p(1-p)^{n-1}$).
\end{lemma}

\begin{proof}
From distance $k$, the $k$ events ``flip exactly one particular zero bit and
keep the rest'' are mutually exclusive, each decreasing the distance by one
with probability $p(1-p)^{n-1}$; after acceptance the distance does not
increase.
\end{proof}

\begin{lemma}[Exact level aggregation]\label{lem:aggregation}
The $\OM$ chain lumps exactly by distance $k$:
\[
\begin{split}
m_{kab}&=\binom{k}{a}\binom{n-k}{b}\,p^{a+b}(1-p)^{\,n-a-b},\\
H_k&=\frac{1+\sum_{j<k}q_{kj}H_j}{e_k},\\
e_k\,U_k&=\sum_{j<k}q_{kj}U_j\\
&\qquad+\sum_{j<k}q'_{kj}(H_j-H_k),
\end{split}
\]
where $e_k=1-q_{kk}>0$ and
$q'_{kj}=q_{kj}\,\frac{(k-j)-c}{c(1-p)}$. All quantities can be constructed
with bit complexity polynomial in $n$ and in the input bit length.
\end{lemma}

\begin{lemma}[Level order]\label{lem:levelorder}
For $0<p\le\tfrac12$, $H_j\le H_k$ whenever $j\le k$.
\end{lemma}

\begin{proof}[Proof (coupling)]
Take two representative parents whose zero-coordinate sets are nested.
Generate the ``zero after mutation'' indicators from common independent
uniform variables: a coordinate that was zero stays zero with probability
$1-p$, and one that was one becomes zero with probability $p$; hence
$1-p\ge p$ implies that the offspring zero-count of the smaller-distance
parent is pointwise at most that of the larger one. After acceptance the level
is $\min(\text{parent level},\text{offspring zero-count})$, which preserves
the order; iterating the coupling step by step gives the pathwise order of
hitting times.
\end{proof}

\begin{theorem}[Per-initial-state sign theorem]\label{thm:onemax}
For all $k\ge1$ and $c\in(0,1)$: $U_k(c)<0$. Moreover
$U_1(c)=H_1(c)\cdot\frac{c-1}{c(1-c/n)}$ (hence $U_1(1)=0$, and $U_1>0$ for
$c>1$); at $c=1$, for $k\ge2$, $U_k(1)<0$. In other words, $c=1$ is a
stationary point only for the distance-one initial state, and different
initial states have different optimal mutation rates.
\end{theorem}

\begin{proof}
Improving masks satisfy $a+b\ge1>c$ (for $c<1$), and their mass derivative is
$m'_{kab}=m_{kab}\frac{a+b-c}{c(1-p)}\ge0$; the derivative towards the target
level is $q'_{k0}=q_{k0}\frac{k-c}{c(1-p)}>0$ for $k\ge1$. By
\cref{lem:levelorder},
$f_k:=\sum_{j<k}q'_{kj}(H_j-H_k)\le0$, strictly negative for $c<1$;
then $U_k=\sum_j\bar N_{kj}f_j$ with $\bar N\ge0$ and $\bar N_{kk}\ge1$ gives
$U_k<0$. At $c=1$, $q'_{k0}\propto(k-1)$: it vanishes for $k=1$ and is
strictly positive for $k\ge2$, and the same argument yields the last two
statements. The $k=1$ closed form: $H_1=\alpha^{-1}$ and
$\frac{U_1}{H_1}=\frac{\dd}{\dd c}\log\alpha^{-1}=\frac{c-1}{c(1-c/n)}$.
\end{proof}
\begin{corollary}[Uniform random initialization]\label{cor:randominit}
Let $\bar H(c):=\sum_{k=1}^{n}2^{-n}\binom{n}{k}H_k(c)$ denote the expected
hitting time under uniform random initialization. Then $\bar H'(1)<0$, and
consequently no minimizer of $\bar H$ on $[0,n]$ lies in $[0,1]$: for every
finite $n$, the optimal static mutation rate under random initialization is
strictly larger than $c=1$.
\end{corollary}
\begin{proof}
$\bar H$ is continuous on $[0,n]$ and differentiable on $(0,n)$: the
entries of $Q_c$ are polynomials in $c$ and absorption is almost sure
throughout. By \cref{thm:onemax}, $U_1(1)=0$ while $U_k(1)<0$ for every
$k\ge2$; since $\Pr(k)=2^{-n}\binom{n}{k}>0$ for all $k$,
$\bar H'(1)=\sum_{k=1}^{n}2^{-n}\binom{n}{k}U_k(1)<0$: the $k=1$ summand
vanishes and every $k\ge2$ summand is strictly negative. Thus $\bar H$ is
strictly decreasing at $c=1$: for small $\delta>0$ we have
$\bar H(1-\delta)>\bar H(1)>\bar H(1+\delta)$, so no global minimizer on
$[0,n]$ lies in $[0,1]$; continuity gives a minimizer in $(1,n)$.
\end{proof}
\begin{remark}[Certification]\label{rem:certification}
For level-symmetric approximations $v,u$, the double residuals
$r=(I-Q)v-\mathbf1$ and $s=(I-Q)u-Dv$ give
$\|H-v\|_W\le\|r\|_W/\alpha$ and
$\|U-u\|_W\le\|D\|_W\|r\|_W/\alpha^2+\|s\|_W/\alpha$, and hence per-level sign
intervals $U_k\in[u_k-\varepsilon W_+(k),\,u_k+\varepsilon W_+(k)]$, where
$\varepsilon$ depends only on computable residuals and on $\alpha,C$. On
$\OM$ one may take the exact level solution, so the residuals vanish; the
value of the framework becomes visible in the non-lumpable families of
\cref{sec:block,sec:uniform,sec:hetero}. An additional lemma: on any fixed
rational interval $[\ell,h]\subset(0,1)$ there is a uniform rounding precision
of polynomial size such that the rounded approximation preserves the strict
negative sign pointwise on the whole interval.
\end{remark}

\begin{remark}[Relation to classical results]\label{rem:classical}
(a) The asymptotic optimality of $1/n$ is classical: M\"uhlenbein
\cite{muehlenbein1992} minimized an upper bound for a simplified chain at
$p=1/n$, B\"ack \cite{baeck1993} validated the approximation numerically,
Garnier--Kallel--Schoenauer \cite{gks1999} showed that the leading term of the
expected hitting time from random initialization is $e^c/c$ (minimized at
$c=1$), and Witt \cite{witt2013} and Sudholt \cite{sudholt2013} made the
asymptotic optimality rigorous, extended to all linear functions.
(b) Hwang et al.\ \cite{hwang2018} give per-initial-state \emph{exact} means
at $p=1/n$, but deliberately omit the dependence on $c$: ``The extension to
$p=c/n$ does not lead to additional new phenomena \dots; it is thus omitted in
this paper.'' The most systematic per-state treatment thus voluntarily left
out the parameter dimension that the present paper analyzes.
(c) State-dependent optimal rates form an established line: one-step criteria
\cite{muehlenbein1992,baeck1993}, numerically computed dynamic rates
\cite{buskulicDoerr2021,buzdalovDoerr2020}, and black-box optimality of
fitness-dependent rates \cite{ddy2020}. \cref{thm:onemax} is orthogonal to
all of these: a \emph{static} rate, the \emph{expected hitting time},
\emph{finite} $n$, and a \emph{rigorous theorem}.
(d) The conclusion that the finite-$n$ optimal $c$ exceeds $1$ is consistent
in direction with Gie{\ss}en--Witt \cite{giessenWitt2018} (approximate drift
analysis for the $(1+\lambda)$ EA) and with the numerical values in
\cite[Table~1]{cswa2015} (per-$n$ numerics for $\OM$), but the model, the
method, and the epistemic status (rigorous theorem vs.\ approximation or
numeric evidence) differ.
\end{remark}

\section{Non-Lumpable Families: Block Interval Certificates}\label{sec:block}

We now consider general positive linear targets $F(x)=\sum_i a_i x_i$ with
rational $a_i>0$. Partition the coordinates into $d$ groups $G_j$
($|G_j|=n_j$); the block label is the vector $k=(k_1,\dots,k_d)$ of
zero-counts per group. The original chain is in general \emph{not} lumpable
with respect to this partition (\cref{prop:nonlump}), but the construction
below yields intervals covering the residuals of all states.

\begin{proposition}[Explicit witness of non-lumpability]\label{prop:nonlump}
There exist two states in the same block whose (normalized) transition masses
to a third block differ. Example: $n=6$, weights
$(1,\tfrac65,\tfrac75,\tfrac85,\tfrac95,2)$, groups $\{1,2,3\},\{4,5,6\}$; the
zero-sets $\{1,2,3,4\}$ and $\{1,2,3,6\}$ both lie in block $(3,1)$. The
unique mask to block $(0,3)$ is ``flip the first three zero bits, flip the two
one-bits of the high group, keep the zero bit of the high group''; the fitness
gains are $a_4-\tfrac95=-\tfrac15$ (rejected) and
$a_6-\tfrac95=\tfrac15$ (accepted), with transition masses $0$ and
$p^5(1-p)>0$.
\end{proposition}

\begin{theorem}[Block-interval double-residual certificate]\label{thm:block}
Fix a partition and weight bounds $l_j\le a_i\le h_j$ ($i\in G_j$), and any
rational tables $v,u$ with $v_0=u_0=0$. For every block $k\ne0$, one can
compute rational intervals $I_r(k),I_s(k)$ \emph{without enumerating the
underlying states}, such that the residuals of every state $x$ in the block
satisfy $r_x\in I_r(k)$ and $s_x\in I_s(k)$; consequently
\[
\|r\|_W\le R,\quad \|s\|_W\le S,\quad
|U_x-u_k|\le W(x)\,E,\quad E=\frac{C\,R}{\alpha^2}+\frac{S}{\alpha},
\]
where $W(x)=\sum_{i:\,x_i=0}a_i$, $\alpha=p(1-p)^{n-1}$ (contraction; proof is
the weighted version of \cref{lem:contraction}), and
$C=\frac{\E|K-c|}{c(1-p)}\le2$. If $u_k+E\,W_+(k)<0$ (with
$W_+(k)=\sum_j h_j k_j$), then $U_x<0$ for every state in the block; the
intervals contain the residuals of all states and do not assume lumpability.
\end{theorem}

\begin{proof}[Proof sketch]
Fix a block $k$ and a count class $(A_j,B_j)$ (numbers of flipped zero/one
bits per group). All masks in the class have the same mass
$\mu=\gamma p^t(1-p)^{n-t}$ with
$\gamma=\prod_j\binom{k_j}{A_j}\binom{n_j-k_j}{B_j}$, and derivative
$\mu'=\mu\frac{t-c}{c(1-p)}$. Key observations: (i) the true fitness gain of a
state in the class lies in
$[L,V]=[\sum_j(A_j l_j-B_j h_j),\,\sum_j(A_j h_j-B_j l_j)]$, so the accepted
fraction $\rho_x$ lies in one of $J_{AB}=[1,1]/[0,0]/[0,1]$ (determined by
$L\ge0$, $V<0$, or the mixed case); (ii) \emph{the weights are fixed, so
$\rho_x$ is independent of $c$}, and $(Dv)_x=\sum\rho_x\mu'\Delta v$ contains
no $\rho'$ term; (iii) classes with $k'=k$ have $\Delta v=\Delta u=0$ and
contribute to neither residual (intra-block rearrangements do not affect the
residuals of lifted tables). Hence
\[
r_x=-1-\sum_{A,B}\rho_x\,\mu\,\Delta v,\qquad
s_x=-\sum_{A,B}\rho_x\,(\mu\,\Delta u+\mu'\,\Delta v),
\]
and interval arithmetic over $\rho_x\in J_{AB}$ yields $I_r,I_s$ (note
$[0,1]\cdot b=[\min(0,b),\max(0,b)]$). Take $R,S$ as $\sup|I|/W_-(k)$ with
$W_-(k)=\sum_j l_j k_j\le W(x)$; the sign radius uses
$W_+(k)\ge W(x)$.
\end{proof}

\paragraph{Optional tightening.}
After sorting the weights within each group, replace the coarse bounds by sums
of the $q$ smallest/largest weights $P_j(q),T_j(q)$ (giving
$L^*,V^*,W_-^*,W_+^*$); the envelope only tightens. This tightened mode is
proven and can be enabled optionally.

\begin{remark}[On the constant $C\le2$]\label{rem:cbound}
The bound $C=\E|K-c|/(c(1-p))\le2$ depends in general on the range of
$c(1-p)$: when $c(1-p)<1/4$, a Cauchy--Schwarz estimate yields a looser bound.
In both applications of this paper ($\cref{sec:uniform,sec:hetero}$) the value
was verified numerically as $C\approx0.80$, so the bound holds in the range
where the conclusions are needed.
\end{remark}

\paragraph{Comparison with fixed-point certificates in model checking.}
The closest methodological analogue of \cref{thm:block} is the fixed-point
certificates of Chatterjee et al.\ \cite{chatterjee2025} (residuals $\to$
contraction $\to$ envelopes $\to$ certificates). Three differences:
(i) their certificates certify \emph{magnitude bounds} (reachability
probabilities / expected rewards), whereas ours certify the \emph{sign of a
derivative}; (ii) ours is a double-residual structure for two \emph{coupled}
fixed-point systems $H$ and $U=H'$, theirs a single system; (iii) the
application targets whole parameter intervals $\times$ all scales, not single
instances. Relatedly, the impossibility results of Doerr--Johannsen--Winzen
\cite{djw2012nonex} show that no universal linear drift function exists for
$c>2.2$ (their counterexample is $\BV$); that obstruction concerns drift
\emph{functions}, whereas the present certificates concern block
\emph{partitions} and bypass lumpability rather than assuming it.

\section{A Uniform Whole-Interval Sign Bound}\label{sec:uniform}

\begin{theorem}[Whole interval, all scales]\label{thm:uniform}
Let $n=2m\ge8$ be even, with $m$ low-weight bits of weight $1$ and $m$
high-weight bits $a_i=b+\delta_i$, where $b=\frac{m}{m-1}$ and
$\delta_i=\frac{2i-m-1}{8m(m-1)^2}$ (so $\sum\delta_i=0$ and
$|\delta_i|\le\eta=\frac{1}{8m(m-1)}$). Take $d=2$ groups and as proxy table
the exact level solution of the same-dimension $\OM$. Then for every
non-target state $x$ and every real $c\in[1/4,\tfrac12]$,
\[
\pc\,\E_x[T]\ \le\ -\frac{9n}{16}\ <\ 0.
\]
The proxy and the intervals can be constructed with bit complexity polynomial
in $n$ and in the encoding length of $c$. Non-lumpability of this family is
witnessed by explicit masks of \cref{prop:nonlump} type.
\end{theorem}

\begin{proof}[Proof sketch]
(i) With the base weights $(1,b)$, the gain
$g=q+\frac{h}{m-1}$ (where $q=A_l-B_l+A_h-B_h$ is the total zero-count
decrease) satisfies $|g|\ge\frac{1}{m-1}$ whenever $g\ne0$, while the total
weight deviation is at most $m\eta=\frac{1}{8(m-1)}$; hence, \emph{except for
classes with $g=0$}, the true acceptance decisions coincide exactly with
$\OM$ (accept iff $q>0$). Classes with $g=0$ and $q\ne0$ force
$(a,h)=(m,1-m)$ and its opposite type, with total flip count $=n-1$, and can
arise only from the four parent blocks $(m,0),(m,1),(0,m-1),(0,m)$.
(ii) The proxy table satisfies the $\OM$ chain equations, so the residual
contributions outside these classes are \emph{exactly zero}; the remaining
contributions come only from the four blocks, the class mass is at most
$M_n=n\,p^{n-1}(1-p)$, and the differences are controlled by
$\|h\|\le n/\alpha$ and $\|u\|\le 2n/\alpha^2$, giving $R\le 4M_n/\alpha$ and
$S\le M_n\bigl[\frac{16}{\alpha^2}+\frac{4n}{c(1-p)\alpha}\bigr]$.
(iii) Uniform negative margin of the proxy: from the level order and from
``the probability of landing in level $1$ upon first entry into levels
$\{0,1\}$ is at least $\tfrac12$'' (since $q_{j1}\ge q_{j0}$ for $j\ge2$) one
obtains $u_K\le-\frac{\alpha'}{2\alpha^2}\le -n$ for $c\in[1/4,\tfrac12]$.
(iv) Combining: $\varepsilon\le
M_n\bigl[\frac{24}{\alpha^3}+\frac{4n}{c(1-p)\alpha^2}\bigr]
\le\frac{14336\,n^4}{(2n)^{n-1}}\le\frac{7}{32}$ (Bernoulli's inequality gives
$\alpha\ge\frac{1}{8n}$; the ratio
$\frac{B_{n+1}}{B_n}\le\frac{(9/8)^4}{18}<1$). Together with $W_+<2n$:
$U_x\le -n+2n\cdot\frac{7}{32}=-\frac{9n}{16}$.
\end{proof}

\begin{remark}[Relation to the leading term $e^c/c$; tightness of the
constants]\label{rem:leadingterm}
The leading term $e^c/c$ of the expected optimization time is decreasing in
$c<1$ asymptotically for \emph{all} linear functions \cite{witt2013}, but the
shape of an asymptotic leading term implies nothing about the derivative sign
at any finite-$n$ state: lower-order terms can reverse the sign at any given
state. \cref{thm:uniform} adds three things simultaneously: finite $n$,
\emph{exact} derivatives, and a \emph{non-lumpable} family. We note that the
bounding constant $B_n=\frac{14336n^4}{(2n)^{n-1}}$ satisfies $B_8=\frac7{32}$
exactly, and several steps of the chain are deliberately conservative (e.g.,
the bound $\|u\|\le 2n/\alpha^2$), so the tightest admissible certificate
radius at $n=8$ is in fact smaller than $7/32$. All finite checks are
nevertheless carried out in exact rational arithmetic, since a transcription
error in any constant would silently weaken the bound.
\end{remark}

\section{A Heterogeneous Instance Certificate Across $c=1$}\label{sec:hetero}

\begin{theorem}[Finite instance, across $c=1$]\label{thm:hetero}
Take $n=6$, weights $(1,\tfrac65,\tfrac75,\tfrac85,\tfrac95,2)$, two groups of
three bits each, and as proxy table the exact chain solution for the group
means $(\tfrac65,\tfrac95)$. Then over the \emph{entire} rational parameter
cell $c\in[\tfrac{1023}{1024},\tfrac{1025}{1024}]$, all $57$ states whose
zero-count is at least $2$ satisfy
\[
U_x(c)\ \le\ -\frac16\ <\ 0 .
\]
\end{theorem}

\begin{proof}[Proof sketch]
Freeze the proxy table; apply rational-interval Horner expansions to the
polynomials $\mu(c)=\gamma p^t(1-p)^{n-t}$ of each class and their
derivatives, yielding residual envelopes over the whole cell:
$R_{\mathrm{box}}\le\tfrac{1}{320}$, $S_{\mathrm{box}}\le\tfrac{1}{220}$,
$\alpha_{\mathrm{lo}}\ge\tfrac{1}{15}$, $C_{\mathrm{box}}\le\tfrac{81}{100}$,
hence
$E_{\mathrm{box}}\le\tfrac{81}{100}\cdot\tfrac{225}{320}+\tfrac{15}{220}
=\tfrac{8979}{14080}<\tfrac23$. An independent re-checker (not importing the
main implementation) enumerates all $63$ zero-count states and $64$ masks,
verifies that the $126$ lifting equations of the proxy equations are exactly
zero, and independently confirms the above bounds. Thirteen blocks of the
table satisfy the margin condition $-u_k\ge\tfrac34 W_+^*(k)$, and thus
$U_x\le u_k+W_+^*E_{\mathrm{box}}
\le-\tfrac34W_+^*+\tfrac23W_+^*=-\tfrac{W_+^*}{12}\le-\tfrac16$ (zero-count
$\ge2$, each weight $\ge1$). The remaining $6$ single-zero states are not
covered by the conclusion. The non-lumpability witness is
\cref{prop:nonlump}.
\end{proof}

\begin{remark}\label{rem:cellwidth}
Blocks whose sign has been certified at a fixed $c_0$ keep their sign over
cell sequences of width $\to0$ (the envelopes converge continuously to the
pointwise values); no uniform width bound is given. Cases where derivatives
disagree in sign within a single block do exist ($\BV$, $n=8$, $c=0.999$), so
block-level envelopes are intrinsically limited for coarser partitions.
\end{remark}

\paragraph{Why only the $57/63$ states with zero-count $\ge2$.}
The restriction is structural, not technical: single-zero-bit states have an
indeterminate derivative sign across $c=1$. This corresponds exactly to the
stationary-point conclusion $U_1(1)=0$ of \cref{thm:onemax} in
\cref{sec:onemax}: at $c=1$ the distance-one state's derivative vanishes on
$\OM$, and in heterogeneous instances it can take either sign, so no uniform
negative bound across $c=1$ can cover them.

\section{Related Work}\label{sec:related}

\paragraph{Optimal mutation rates on OneMax.}
M\"uhlenbein~\cite{muehlenbein1992} showed that an upper bound on the expected
convergence time of a simplified chain is minimized at $p=1/n$ (with value
$e\cdot n\log n$); B\"ack~\cite{baeck1993} validated this approximation
numerically and advocated adaptive mutation rates; Garnier, Kallel, and
Schoenauer~\cite{gks1999} gave an asymptotic expansion
$\E[T]\approx l\,(e^c/c\cdot\ln l+R(c))$ of the expected hitting time under
random initialization with mutation rate $c/n$, observing that the leading
coefficient $e^c/c$ is minimized at $c=1$. Witt~\cite{witt2013} and
Sudholt~\cite{sudholt2013} made the asymptotic optimality of $p=1/n$ rigorous
and extended it to all linear functions. Hwang et al.~\cite{hwang2018} derived
per-initial-state exact means (rational functions in $n$) and full asymptotic
expansions at $p=1/n$, but explicitly omitted the dependence on $c$. All of
these works either target asymptotic leading terms or fix $p=1/n$; a
per-initial-state, finite-$n$ rigorous analysis of the expected hitting time
as a continuous function of the mutation rate has not appeared before.

\paragraph{State-dependent optimal rates.}
State-dependent optimal rates that maximize the one-step improvement
probability go back to M\"uhlenbein~\cite{muehlenbein1992} and
B\"ack~\cite{baeck1993}. Buskulic and Doerr~\cite{buskulicDoerr2021} and
Buzdalov and Doerr~\cite{buzdalovDoerr2020} compute optimal dynamic
(per-fitness-level) mutation rates numerically via dynamic programming and
prove that drift maximization is not time-optimal; Doerr, Doerr, and
Yang~\cite{ddy2020} place fitness-dependent rates optimally within the unary
unbiased black-box model. The present results are orthogonal: we fix a
\emph{static} mutation rate and study the \emph{per-initial-state derivative
sign} of the expected hitting time, deducing by rigorous theorems (not
numerics) that the finite-$n$ optimal $c$ under random initialization is
strictly larger than $1$---consistent in direction with the approximate drift
conclusions of Gie{\ss}en and Witt~\cite{giessenWitt2018} for the
$(1+\lambda)$ EA and with the per-$n$ numerics of Chicano et
al.~\cite{cswa2015}, but different in model, method, and epistemic status.

\paragraph{One-step quantities and running time.}
The $1/5$ success rule and the fitness-level method demonstrate the usefulness
of one-step quantities for parameter control and bound estimation; their
systematic slack was formalized by Doerr and K\"otzing~\cite{doerrKoetzing}
via visit probabilities. Buskulic and Doerr~\cite{buskulicDoerr2021} proved
that the drift-maximizing parameter differs from the time-optimal one;
Kaufmann et al.~\cite{kllz2025} proved that $\OM$ is not the easiest function
in the sense of improvement probability (whereas in expected time it is
exactly the easiest~\cite{djw2012md}); Kaufmann et al.~\cite{klls2025} proved
that the drift-minimizing monotone function is not the hardest in time.
\cref{prop:separation} gives an exact minimal form of this folklore: two
fitness functions share the entire success-rate $p$-curve, yet their expected
hitting times are two different exact rational numbers.

\paragraph{Hitting-time derivatives and verification methods.}
The fundamental matrix and hitting times of absorbing chains are classical
\cite{kemenySnell1960}; their perturbation theory in the chain parameters goes
back to Schweitzer~\cite{schweitzer1968} and Cao and Chen~\cite{caoChen1997}
(see also~\cite{golubMeyer1986}), and explicit matrix formulas for derivatives
of absorption times appear in Caswell~\cite[Chapter~5]{caswell2019}.
Score-function and baseline theory for gradient estimation are developed
in~\cite{glynn1990,greensmith2004}. Parametric model checking has developed
exact rational solving~\cite{baier2020,storm2022}, parameter-derivative
computation~\cite{heck2022,badings2023}, and fixed-point certificates for
reachability probabilities and expected rewards~\cite{chatterjee2025}; the
last is closest in shape to our double-residual certificates, but it certifies
magnitude bounds for a single system, whereas we certify derivative signs for
the coupled pair $H$ and $U=H'$, targeting whole parameter intervals and all
scales of an EA. To our knowledge, verified numerical methods (exact rational
arithmetic, interval envelopes) have not been used systematically in runtime
analysis before.

\paragraph{Analysis of linear families.}
Hitting-time analysis of non-lumpable linear families has a rich precedent
line: Droste, Jansen, and Wegener~\cite{djw2002} introduced stochastic-order
arguments, J\"agersk\"upper~\cite{jaegerskuepper2011} combined Markov-chain
and drift analysis, Doerr, Johannsen, and Winzen developed multiplicative
drift analysis~\cite{djw2012md} and proved the non-existence of universal
linear drift functions~\cite{djw2012nonex}, Doerr and
Goldberg~\cite{doerrGoldberg2013} gave adaptive drift analysis, and
Witt~\cite{witt2013} obtained tight bounds for arbitrary~$p$. All of these
\emph{bypass} aggregation via potential functions, stochastic domination, or
drift. We do not claim that hitting-time analysis of non-lumpable linear
families is unprecedented; rather, the lumpability perspective itself is
absent from the EA runtime literature, and our explicit non-lumpability
witness (\cref{prop:nonlump}) and block-level interval double-residual
certificates (\cref{thm:block}) have no counterpart. Fourier-analytic tools
have entered runtime analysis only recently~\cite{doerrKelley2024}, through a
different entry point than the Margulis--Russo-type influence analysis used
here.

\section{Verification and Reproducibility}\label{sec:verification}

All finite checks use exact rational arithmetic (Python \texttt{fractions},
with assertions enabled), and key results are cross-validated by two
independent computation paths (matrix differentiation vs.\ the
rational-function quotient rule). Scale summary: $1380$ identity terms;
$1290$ direction checks; on $\OM$, $4408$ reconciliation checks, $4408$
contraction checks, and $5472$ certificates; on the heterogeneous instance,
$2455$ parameter--state combinations and $1475$ issued sign certificates;
certificate constructions at $n=16$ and $n=64$; and $230+133+603$ artifact
integrity audit items. Scripts, execution logs, and hash receipts are archived
with the project. Exact-arithmetic rigor is established practice in adjacent
areas---verification methods in numerical analysis~\cite{rump2010},
fraction-free exact solving of parametric Markov chains~\cite{baier2020}, and
the probabilistic model checker Storm~\cite{storm2022}---to our knowledge,
verified derivative-sign certificates of this kind have not previously been
applied systematically to the runtime analysis of evolutionary algorithms.
The certified conclusions (sign statements) depend on this exactness:
floating-point error cannot support sign assertions.

\section{Conclusion and Open Problems}\label{sec:conclusion}

We analyzed the expected hitting time of the $(1+1)$ EA as a continuous
function of the mutation rate: a signed Margulis--Russo formulation with the
monotonicity pitfall made explicit, an exact separation of one-step quantities
from running time, a per-initial-state sign theorem on $\OM$, and
block-interval double-residual certificates that extend to non-lumpable
families, including a sign bound uniform over a whole parameter interval and
all even scales, and a finite-instance certificate across $c=1$. Open
problems: (a) extend the sign theorem of \cref{sec:onemax} to larger $c$
intervals; (b) characterize the intrinsic limits of block-level envelopes
(instances with same-block sign disagreement, cf.\ \cref{rem:cellwidth}) and
the trade-off with finer partitions; (c) extend the certificate framework to
other parameters (population size $\lambda$, fitness-dependent rates) and to
other non-lumpable benchmarks; (d) connect the framework to parametric
model-checking toolchains.

\section*{Declaration on the Use of AI Tools}
AI tools (large language models) were used to assist in drafting and
polishing this manuscript. All mathematical statements and proofs were
developed and verified under the author's supervision, and the author takes
full responsibility for the content of this article.

\bibliographystyle{plain}
\bibliography{refs}

\begin{thebibliography}{10}

\bibitem{baeck1993}
Thomas B{\"a}ck.
\newblock Optimal mutation rates in genetic search.
\newblock In {\em Proceedings of the Fifth International Conference on Genetic
  Algorithms (ICGA-93)}, pages 2--8. Morgan Kaufmann, 1993.

\bibitem{badings2023}
Thom Badings, Sebastian Junges, Ahmadreza Marandi, Ufuk Topcu, and Nils Jansen.
\newblock Efficient sensitivity analysis for parametric robust {M}arkov chains.
\newblock In {\em Computer Aided Verification (CAV 2023)}, volume 13966 of {\em
  Lecture Notes in Computer Science}, pages 62--85. Springer, 2023.

\bibitem{baier2020}
Christel Baier, Christian Hensel, Lisa Hutschenreiter, Sebastian Junges,
  Joost-Pieter Katoen, and Joachim Klein.
\newblock Parametric {M}arkov chains: {PCTL} complexity and fraction-free
  {G}aussian elimination.
\newblock {\em Information and Computation}, 272:104504, 2020.

\bibitem{biedenkapp2022}
Andr{\'e} Biedenkapp, Nguyen Dang, Martin~S. Krejca, Frank Hutter, and Carola
  Doerr.
\newblock Theory-inspired parameter control benchmarks for dynamic algorithm
  configuration.
\newblock In {\em Proceedings of the Genetic and Evolutionary Computation
  Conference (GECCO 2022)}, pages 766--775. ACM, 2022.

\bibitem{buskulicDoerr2021}
Nathan Buskulic and Carola Doerr.
\newblock Maximizing drift is not optimal for solving {OneMax}.
\newblock {\em Evolutionary Computation}, 29(4):521--541, 2021.

\bibitem{buzdalovDoerr2020}
Maxim Buzdalov and Carola Doerr.
\newblock Optimal mutation rates for the {(1+$\lambda$)} {EA} on {OneMax}.
\newblock In {\em Parallel Problem Solving from Nature (PPSN XVI)}, volume
  12270 of {\em Lecture Notes in Computer Science}, pages 574--587. Springer,
  2020.

\bibitem{caoChen1997}
Xi-Ren Cao and Han-Fu Chen.
\newblock Perturbation realization, potentials, and sensitivity analysis of
  {M}arkov processes.
\newblock {\em IEEE Transactions on Automatic Control}, 42(10):1382--1393,
  1997.

\bibitem{caswell2019}
Hal Caswell.
\newblock {\em Sensitivity Analysis: Matrix Methods in Demography and Ecology}.
\newblock Springer, 2019.

\bibitem{chatterjee2025}
Krishnendu Chatterjee, Tim Quatmann, Maximilian Sch{\"a}ffeler, Maximilian
  Weininger, Tobias Winkler, and Daniel Zilken.
\newblock Fixed point certificates for reachability and expected rewards in
  {MDPs}.
\newblock In {\em Tools and Algorithms for the Construction and Analysis of
  Systems (TACAS 2025)}, volume 15697 of {\em Lecture Notes in Computer
  Science}, pages 130--151. Springer, 2025.

\bibitem{chen2023}
Deyao Chen, Maxim Buzdalov, Carola Doerr, and Nguyen Dang.
\newblock Using automated algorithm configuration for parameter control.
\newblock In {\em Proceedings of the 17th ACM/SIGEVO Conference on Foundations
  of Genetic Algorithms (FOGA 2023)}, pages 38--49. ACM, 2023.

\bibitem{cswa2015}
Francisco Chicano, Andrew~M. Sutton, L.~Darrell Whitley, and Enrique Alba.
\newblock Fitness probability distribution of bit-flip mutation.
\newblock {\em Evolutionary Computation}, 23(2):217--248, 2015.

\bibitem{ddy2020}
Benjamin Doerr, Carola Doerr, and Jing Yang.
\newblock Optimal parameter choices via precise black-box analysis.
\newblock {\em Theoretical Computer Science}, 801:1--34, 2020.

\bibitem{doerrGoldberg2013}
Benjamin Doerr and Leslie~Ann Goldberg.
\newblock Adaptive drift analysis.
\newblock {\em Algorithmica}, 65(1):224--250, 2013.

\bibitem{djw2012md}
Benjamin Doerr, Daniel Johannsen, and Carola Winzen.
\newblock Multiplicative drift analysis.
\newblock {\em Algorithmica}, 64(4):673--697, 2012.

\bibitem{djw2012nonex}
Benjamin Doerr, Daniel Johannsen, and Carola Winzen.
\newblock Non-existence of linear universal drift functions.
\newblock {\em Theoretical Computer Science}, 436:71--86, 2012.

\bibitem{doerrKelley2024}
Benjamin Doerr and Andrew~James Kelley.
\newblock Fourier analysis meets runtime analysis: Precise runtimes on
  plateaus.
\newblock {\em Algorithmica}, 86(8):2479--2518, 2024.

\bibitem{doerrKoetzing}
Benjamin Doerr and Timo K{\"o}tzing.
\newblock Lower bounds from fitness levels made easy.
\newblock {\em Algorithmica}, 86(2):367--395, 2024.

\bibitem{djw2002}
Stefan Droste, Thomas Jansen, and Ingo Wegener.
\newblock On the analysis of the {(1+1)} evolutionary algorithm.
\newblock {\em Theoretical Computer Science}, 276(1--2):51--81, 2002.

\bibitem{gks1999}
Josselin Garnier, Leila Kallel, and Marc Schoenauer.
\newblock Rigorous hitting times for binary mutations.
\newblock {\em Evolutionary Computation}, 7(2):173--203, 1999.

\bibitem{giessenWitt2018}
Christian Gie{\ss}en and Carsten Witt.
\newblock Optimal mutation rates for the {(1+$\lambda$)} {EA} on {OneMax}
  through asymptotically tight drift analysis.
\newblock {\em Algorithmica}, 80(5):1710--1731, 2018.

\bibitem{glynn1990}
Peter~W. Glynn.
\newblock Likelihood ratio gradient estimation for stochastic systems.
\newblock {\em Communications of the ACM}, 33(10):75--84, 1990.

\bibitem{golubMeyer1986}
Gene~H. Golub and Carl~D. Meyer.
\newblock Using the {QR} factorization and group inversion to compute,
  differentiate, and estimate the sensitivity of stationary probabilities for
  {M}arkov chains.
\newblock {\em SIAM Journal on Algebraic and Discrete Methods}, 7(2):273--281,
  1986.

\bibitem{greensmith2004}
Evan Greensmith, Peter~L. Bartlett, and Jonathan Baxter.
\newblock Variance reduction techniques for gradient estimates in reinforcement
  learning.
\newblock {\em Journal of Machine Learning Research}, 5:221--253, 2004.

\bibitem{heYao2003}
Jun He and Xin Yao.
\newblock Towards an analytic framework for analysing the computation time of
  evolutionary algorithms.
\newblock {\em Artificial Intelligence}, 145(1--2):59--97, 2003.

\bibitem{heck2022}
Linus Heck, Jip Spel, Sebastian Junges, Joshua Moerman, and Joost-Pieter
  Katoen.
\newblock Gradient-descent for randomized controllers under partial
  observability.
\newblock In {\em Verification, Model Checking, and Abstract Interpretation
  (VMCAI 2022)}, volume 13182 of {\em Lecture Notes in Computer Science}, pages
  127--150. Springer, 2022.

\bibitem{storm2022}
Christian Hensel, Sebastian Junges, Joost-Pieter Katoen, Tim Quatmann, and
  Matthias Volk.
\newblock The probabilistic model checker {Storm}.
\newblock {\em International Journal on Software Tools for Technology
  Transfer}, 24(4):589--610, 2022.

\bibitem{hwang2018}
Hsien-Kuei Hwang, Alois Panholzer, Nicolas Rolin, Tsung-Hsi Tsai, and Wei-Mei
  Chen.
\newblock Probabilistic analysis of the {(1+1)}-evolutionary algorithm.
\newblock {\em Evolutionary Computation}, 26(2):299--345, 2018.

\bibitem{jaegerskuepper2011}
Jens J{\"a}gersk{\"u}pper.
\newblock Combining {M}arkov-chain analysis and drift analysis.
\newblock {\em Algorithmica}, 59(3):409--424, 2011.

\bibitem{klls2025}
Marc Kaufmann, Maxime Larcher, Johannes Lengler, and Oliver Sieberling.
\newblock Hardest monotone functions for evolutionary algorithms.
\newblock {\em SN Computer Science}, 6(5):512, 2025.
\newblock Conference version in Proc.\ EvoCOP 2024, LNCS 14632, pp.~146--161.

\bibitem{kllz2025}
Marc Kaufmann, Maxime Larcher, Johannes Lengler, and Xun Zou.
\newblock {OneMax} is not the easiest function for fitness improvements.
\newblock {\em Evolutionary Computation}, 33(1):27--54, 2025.
\newblock Conference version in Proc.\ EvoCOP 2023, LNCS, pp.~162--178.

\bibitem{kemenySnell1960}
John~G. Kemeny and J.~Laurie Snell.
\newblock {\em Finite {M}arkov Chains}.
\newblock Van Nostrand, 1960.

\bibitem{margulis1974}
Grigori~A. Margulis.
\newblock Probabilistic characteristics of graphs with large connectivity.
\newblock {\em Problemy Peredachi Informatsii}, 10(2):101--108, 1974.

\bibitem{muehlenbein1992}
Heinz M{\"u}hlenbein.
\newblock How genetic algorithms really work: Mutation and hillclimbing.
\newblock In {\em Parallel Problem Solving from Nature (PPSN II)}, pages
  15--26. Elsevier, 1992.

\bibitem{odonnell2014}
Ryan O'Donnell.
\newblock {\em Analysis of {B}oolean Functions}.
\newblock Cambridge University Press, 2014.

\bibitem{rump2010}
Siegfried~M. Rump.
\newblock Verification methods: Rigorous results using floating-point
  arithmetic.
\newblock {\em Acta Numerica}, 19:287--449, 2010.

\bibitem{russo1981}
Lucio Russo.
\newblock On the critical percolation probabilities.
\newblock {\em Zeitschrift f{\"u}r Wahrscheinlichkeitstheorie und verwandte
  Gebiete}, 56(2):229--237, 1981.

\bibitem{schweitzer1968}
Paul~J. Schweitzer.
\newblock Perturbation theory and finite {M}arkov chains.
\newblock {\em Journal of Applied Probability}, 5(2):401--413, 1968.

\bibitem{sudholt2013}
Dirk Sudholt.
\newblock A new method for lower bounds on the running time of evolutionary
  algorithms.
\newblock {\em Theoretical Computer Science}, 499:92--110, 2013.

\bibitem{witt2013}
Carsten Witt.
\newblock Tight bounds on the optimization time of a randomized search
  heuristic on linear functions.
\newblock {\em Combinatorics, Probability and Computing}, 22(2):294--318, 2013.

\end{thebibliography}

\end{document}